\documentclass[preprint,12pt]{elsarticle}

\usepackage{amssymb}
\usepackage{amsmath}
\usepackage{amsthm}

\usepackage{booktabs}
\usepackage{bbm}
\usepackage{bm}
\usepackage{mathrsfs}
\usepackage{diagbox}
\usepackage{graphicx}
\usepackage{booktabs}
\usepackage{subfigure}
\usepackage{makecell}
\usepackage[colorlinks=true,linkcolor=blue,urlcolor=blue,citecolor=blue]{hyperref}

\newtheorem{theorem}{Theorem}

\journal{Knowledge-Based Systems}

\begin{document}

\begin{frontmatter}



\title{HeGMN: Heterogeneous Graph Matching Network for Learning Graph Similarity}

\author[label1]{Shilong Sang}
\ead{1022040812@njupt.edu.cn}
\author[label1,label2,label3]{Ke-Jia Chen\corref{cor1}}
\ead{chenkj@njupt.edu.cn}
\author[label1,label2,label3]{Zheng Liu}
\ead{zliu@njupt.edu.cn}
\cortext[cor1]{Correspondence to: No.9 Wenyuan Road, Nanjing, Jiangsu, China.}
\affiliation[label1]{
            organization={School of Computer Science},
            addressline={Nanjing University of Posts and Telecommunications},
            city={Nanjing},
            postcode={210023},
            state={Jiangsu},
            country={China}
            }
\affiliation[label2]{
            organization={Jiangsu Key Laboratory of Big Data Security \& Intelligent Processing},
            addressline={Nanjing University of Posts and
Telecommunications},
            city={Nanjing},
            postcode={210023},
            state={Jiangsu},
            country={China}
            }
\affiliation[label3]{
            organization={State Key Laboratory of Novel Software Technology},
            addressline={Nanjing University},
            city={Nanjing},
            postcode={210023},
            state={Jiangsu},
            country={China}
            }
            
\begin{abstract}
Graph similarity learning (GSL), also referred to as graph matching in many scenarios, is a fundamental problem in computer vision, pattern recognition, and graph learning. However, previous GSL methods assume that graphs are homogeneous and struggle to maintain their performance on heterogeneous graphs. To address this problem, this paper proposes a Heterogeneous Graph Matching Network (HeGMN), which is an end-to-end graph similarity learning framework composed of a two-tier matching mechanism. Firstly, a heterogeneous graph isomorphism network is proposed as the encoder, which reinvents graph isomorphism network for heterogeneous graphs by perceiving different semantic relationships during aggregation. Secondly, a graph-level and node-level matching modules are designed, both employing type-aligned matching principles. The former conducts graph-level matching by node type alignment, and the latter computes the interactions between the cross-graph nodes with the same type thus reducing noise interference and computational overhead. Finally, the graph-level and node-level matching features are combined and fed into fully connected layers for predicting graph similarity scores. In experiments, we propose a heterogeneous graph resampling method to construct heterogeneous graph pairs and define the corresponding heterogeneous graph edit distance, filling the gap in missing datasets. Extensive experiments demonstrate that HeGMN consistently achieves advanced performance on graph similarity prediction across all datasets.
\end{abstract}


\end{frontmatter}

\section{Introduction}
Graph-structure data is present in various applications, such as molecular compound networks, citation networks, and social networks. Graph similarity learning, also referred to as graph matching, aims to measure the similarities between a pair of graph-structured objects, stands out as one of the most prominent topics in the field of graph machine learning. It plays a pivotal role in numerous foundational tasks such as drug-drug interaction prediction (DDI) \cite{cardoso2020collection, coupry2022application}, code similarity prediction \cite{dai2023study}, malicious program detection \cite{noble2003graph, wang2019heterogeneous}, expert community network localization \cite{li2023semi}, cross-domain knowledge transfer \cite{ebsch2020using, qi2021unsupervised, wu2020neighborhood}, image semantic matching \cite{guo2018neural}, and so on.

Graph edit distance (GED) and maximum common subgraph (MCS) are two traditional measures for graph similarity calculation, which can be calculated by heuristic methods \cite{EJ2024ged, fankhauser2011speeding, riesen2009approximate}. However, these methods are computationally expensive and known as NP-hard. To overcome this limitation, current solutions formulate graph similarity computation into a learning problem. A common practice is that an inductive graph neural network (GNN) model is trained on pairs of graphs and then equipped with a matching mechanism to approximate graph similarity scores \cite{bai2019simgnn, tan2023exploring}. The trained model can later predict similarity estimation for unseen graph pairs. 

However, real-world graphs are usually heterogeneous, wherein both nodes and edges exhibit distinct semantics. Existing GSL methods predominantly conduct research in a homogeneous graph environment, which may result in suboptimal performance when applied to heterogeneous graphs, as the matching result is highly correlated to the element types. As shown in Fig.~\ref{fig:1}, the molecule on the left is ethylene, and the molecule on the right is difluoroethylene. Ethylene contains a carbon-carbon double bond and four hydrogen atoms, while difluoroethylene replaces two hydrogen atoms with two fluorine atoms. In homogeneous graph scenarios, ethylene and difluoroethylene have the same structure. In contrast, our model can also differentiate the atoms (i.e., node types) in heterogeneous graph scenarios. To verify this hypothesis, we conducted a preliminary experiment by constructing a heterogeneous graph dataset ACM1000 (see Section \ref{subsection1} for experiment details) and comparing the performance of several GSL baselines on ACM1000 and a benchmark homogeneous graph dataset AIDS700nef. The experimental results (Fig.~\ref{fig:pre}) indicate that all GSL baselines perform much worse on ACM1000 than on AIDS700nef, with a rise ranging from 25\% to 76\% in the MSE indicator.
\begin{figure}
    \centering
    \includegraphics[width=0.6\linewidth]{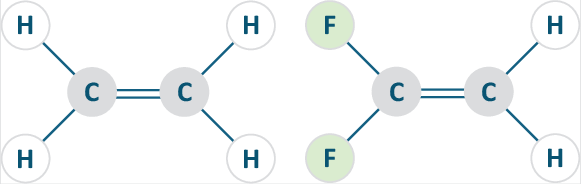}
    \caption{Molecular structures of ethylene (left), a two-carbon molecule with one double bond and four hydrogen atoms ($C_2H_4$), and difluoroethylene (right), a derivative of ethylene where two hydrogen atoms are replaced by fluorine atoms ($C_2H_2F_2$).}
    \label{fig:1}
\end{figure}

\begin{figure}[!htbp]
  \centering
  \includegraphics[width=7cm]{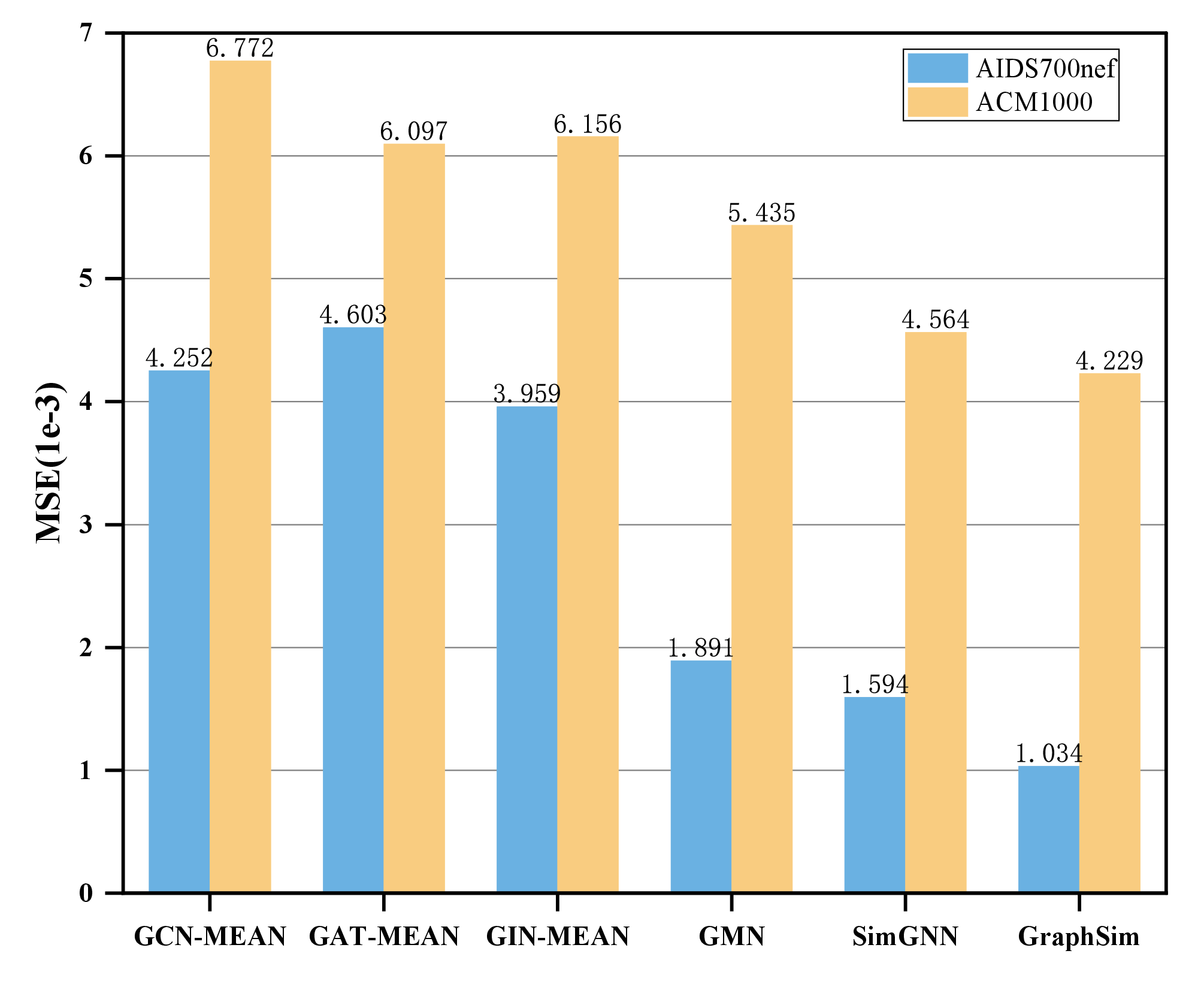}  
  \caption{Comparison of several GSL baselines on a homogeneous graph dataset (AIDS700nef) and a heterogeneous graph dataset (ACM1000).}
  \label{fig:pre}
  % \vspace{-10pt}
\end{figure}

To address the aforementioned challenges, this paper proposes an end-to-end heterogeneous graph similarity learning (HGSL) framework, namely, Heterogeneous Graph Matching Network (HeGMN). To start with, the HGSL datasets are constructed, where we first randomly sample subgraphs with different sizes from real heterogeneous graphs using the BFS algorithm and then build graph pairs from the sampling pool. Meanwhile, the heterogeneous graph edit distance (HGED) is defined to quantify the degree of similarity between each graph pair. For constructing the HeGMN, a heterogeneous graph isomorphism network (HGIN) is proposed, which reinvents GIN by leveraging node type information during the message passing process to better aggregate the diverse structures and semantics. Subsequently, a type-aligned graph matching algorithm is proposed to compute type embeddings for graph-level matching, and a type-aligned node matching algorithm is proposed to calculate the matching degree between cross-graph node pairs with the same type based on an interaction attention mechanism. Finally, the node-level and the graph-level matching features are concatenated and fed to fully connected layers to predict the HGED value.

Overall, Our main contributions are summarized as follows:
\begin{itemize}
    \item To the best of our knowledge, we are the first to investigate the deep GSL method in heterogeneous graph scenarios, and propose an end-to-end HGSL model named HeGMN.
    \item The definition of heterogeneous graph edit distance is proposed to measure the matching degree between two heterogeneous graphs. Additionally, the datasets tailored specifically for the HGSL problem are constructed.
    \item In HeGMN, a heterogeneous graph isomorphism network is designed to effectively learn structures involving different semantic types, yielding representations more suitable for the GSL task. More critically, a type-aligned node-level matching method is proposed, which improves the matching precision and efficiency.
    \item Extensive experiments demonstrate the superiority of the proposed model compared with existing GSL baselines. The effectiveness of each proposed module is also verified by ablation analysis.
\end{itemize}

\section{Related Work}
\subsection{Graph Similarity Learning}
Traditional graph similarity calculation methods are based on graph isomorphism or other structural similarity measures, such as graph edit distance (GED) \cite{riesen2013novel, bunke1983distance} and maximum common subgraph (MCS) \cite{bunke1998graph}. However, these methods are NP-hard and have limitations in real-world applications. Subsequently, kernel-based graph similarity measures emerged \cite{borgwardt2005shortest, yan2005substructure, yoshida2019learning}, which employ kernel functions like random walks to compute the embeddings for individual graphs, and then calculate the matching degree through inner product operation. Nevertheless, kernel functions are manually designed and cannot autonomously learn graph features.

With the prevalence of graph neural networks (GNNs), a series of deep graph similarity learning (GSL) methods have been developed. Initial methods (i.e., GCN-MEAN \cite{kipf2016semi}, GIN-MEAN \cite{xu2018powerful}, etc.) utilize GNNs and mean pooling to represent graphs, followed by fully connected layers to obtain similarity scores. However, these methods lack fine-grained matching at the node level. SimGNN \cite{bai2019simgnn} goes beyond graph-level comparisons by pioneering the utilization of histograms to feature the similarity matrix of the cross-graph node pairs. Later, GraphSim \cite{bai2020learning} improves SimGNN by addressing the non-differentiability issue in the histogram calculation, which preserves node positional information through node ranking and uses CNN to fuse matching matrices across different layers. GMN \cite{li2019graph} incorporates cross-graph node interaction attention to obtain richer representations, enabling more precise matching. In addition to improving the node representation, the attention mechanism is subsequently used in the matching matrix. NAGSL \cite{tan2023exploring} calculates the inter-graph node pair attention and then uses self-attention mechanism to align the node matching matrix for similarity learning. Contrastive learning is also introduced in graph similarity learning. CGMN \cite{jin2022cgmn} employs a graph augmentation method to create views for input graphs, and then utilizes a contrastive loss to minimize the distance between embeddings of similar graphs. CGSim \cite{wang2023contrastive} proposes a dual-contrastive learning module for both node-graph and graph-graph matching mechanisms. Due to the high time consumption associated with node-level matching, some methods strive to balance between matching accuracy and efficiency. MGMN \cite{ling2021multilevel} reduces the time consumption by calculating the similarity between node-graph pairs instead of node pairs. ERIC \cite{zhuo2022efficient} expedites the training phase through alignment regularization and then directly employs the learned graph-level representations for similarity calculation during the inference phase. GraSP \cite{zheng2024grasp} enhances node features using positional encoding instead of cross-graph node interactions, thereby avoiding the significant computational overhead.

Despite the above achievements, existing GSL methods are primarily oriented to homogeneous graphs, without noticing the impact of element types on graph matching. When developing deep heterogeneous GSL methods, the foremost issue is how to represent heterogeneous graphs.

\subsection{Heterogeneous Graph Representation Learning}
The early heterogeneous graph representation learning often uses meta-paths to model semantic relationships. In metapath2vec \cite{dong2017metapath2vec}, nodes of different types are mapped to a low-dimensional vector space through random walks. HAN \cite{wang2019heterogeneous} combines attention mechanism and multi-layer neural networks based on meta-paths to capture the importance at both node and semantic levels. GTN \cite{yun2019graph} autonomously discovers valuable meta-paths without manual selection. However, the methods based on meta-path have high spatiotemporal complexity, so subsequent works explore alternative methods. HetGNN \cite{zhang2019heterogeneous} generates neighbors for nodes based on random walks, and then leverages Bi-LSTM to aggregate node features for each type, yielding more precise representations. Lv et al. \cite{lv2021we} reexamined previous models and argued that meta-paths are not obligatory for heterogeneous graphs. A slight modification of graph attention network (GAT) \cite{velivckovic2018graph}  to its heterogeneous version (i.e., Simple-HGN) can outperform previous methods. Inspired by Simple-HGN, RE-GNN \cite{wang2023enabling} uses a single parameter to represent the importance of each relation type and introduces self-loops related to node types, enabling GNNs to handle heterogeneous graphs.

However, none of the above methods are specifically designed for the HGSL task, where both structural and semantic features of heterogeneous graphs are equally important. This paper explores a new representation learning method and a new matching method tailored for heterogeneous graphs.

\section{Preliminaries}
This section introduces basic terminology and problem description involved in the method. The important notations in this paper can be found in Table~\ref{tab:notations}.

\begin{table}[t]\small
    \centering
    \caption{Important notations.}
    \begin{tabular}{c c}
    \toprule
    Symbol & Definition \\
    \midrule
        $\boldsymbol{G}$ & heterogeneous graph\\
        $\Phi$ & node type mapping function \\
        $\Psi$ & edge type mapping function \\
        $\boldsymbol{C}$ & the set of node types \\
        $\boldsymbol{R}$ & the set of edge types \\
        $i, j$ & graph index \\
        $n, m$ & node index \\
        $\boldsymbol{X}$ & initial node feature matrix \\
        $\boldsymbol{Z}$ & node embedding matrix learned by HGIN \\
        $\boldsymbol{T}$ & type embedding matrix \\
        $\bm{t_{i, c}}$ & embedding of type $c$ in graph $\boldsymbol{G_i}$ \\
        $\mathbbm{I}(\cdot)$ & an indicator function \\
        $\mathcal{N}(n)$ & the neighbor set of node $n$ \\
        $\boldsymbol{W}$ & a learnable matrix \\
        $\boldsymbol{S}$ & similarity matrix \\
        $\rho$ & Spearman’s rank correlation coefficient \\
        $\tau$ & Kendall’s rank correlation coefficient \\
    \bottomrule
    \end{tabular}
    \label{tab:notations}
\end{table}

\subsection{Heterogeneous Graph}
A heterogeneous graph is defined as $\boldsymbol{G} = (\boldsymbol{V}, \boldsymbol{E}, \boldsymbol{X}, \Phi, \Psi)$. Here, $\boldsymbol{V}$ represents the set of nodes with a node type mapping function $\Phi: \boldsymbol{V} \to \boldsymbol{C}$ and $\boldsymbol{E}$ is the set of edges with an edge type mapping function $\Psi: \boldsymbol{E} \to \boldsymbol{R}$, where $\boldsymbol{C}$ and $\boldsymbol{R}$ denote the predefined set of node types and edge types, respectively, and $|\boldsymbol{C}| + |\boldsymbol{R}| > 2$ since $\boldsymbol{G}$ is heterogeneous. Each node $v_n \in \boldsymbol{V}$ has a type $c_n = \Phi(v_n) \in \boldsymbol{C}$, and each edge $e_{n,m} \in \boldsymbol{E}$ has a type $r_{n, m} = \Psi(e_{n,m}) \in \boldsymbol{R}$. $\boldsymbol{X}$ denotes the feature matrix of the nodes.

\subsection{Heterogeneous Graph Edit Distance}
Graph Edit Distance (GED) measures the minimum cost of edit operations to transform one graph into another and has been widely applied in graph similarity search \cite{bai2019simgnn, chang2022accelerating}, graph classification \cite{wu2020neighborhood}, and graph matching \cite{jin2022cgmn}. This paper extends the concept of GED to heterogeneous graph scenarios, defining Heterogeneous Graph Edit Distance (HGED).    

The GED between heterogeneous graphs $\boldsymbol{G_i}$ and $\boldsymbol{G_j}$ ($\boldsymbol{G_i}, \boldsymbol{G_j} \in \boldsymbol{G}$), denoted as $\operatorname{HGED}(\boldsymbol{G_i}, \boldsymbol{G_j})$, refers to the minimum cost of edit operations required to transform $\boldsymbol{G_i}$ into $\boldsymbol{G_j}$.
\begin{eqnarray}\label{eq:HGED}
\operatorname{HGED}(\boldsymbol{G_i}, \boldsymbol{G_j}) = \mathop{min}_{(edit_1, ..., edit_{\mathscr{L}}) \in \gamma(\boldsymbol{G_i}, \boldsymbol{G_j})}
\sum^{\mathscr{L}}_{\ell = 1}\operatorname{cost}(edit_{\ell})
\end{eqnarray}
where $\gamma(\boldsymbol{G_i},\ \boldsymbol{G_j})$ represents the set of all edit paths, and $\operatorname{cost}(edit_\ell)$ measures the cost of the edit operation $edit_{\ell}$. In addition to adding/removing nodes/edges, the edit operation on heterogeneous graphs also involves \textit{modifying types}, as both nodes and edges have distinct types. However, it is arbitrary to regard \textit{modifying types} as one single edit, since the type represents the semantics of elements and often requires two edit operations for alignment. In heterogeneous graphs, node/edge types constitute essential semantic constraints. Unlike homogeneous GED where type preservation is implicit, modifying an element's type in HGED requires explicit edit operations: (1) deleting the original element with its legacy type, and (2) inserting a new element with the target type. This two-step process ensures semantic consistency, as type modifications cannot be achieved through in-place alteration under the standard graph edit distance framework \cite{bunke1983distance}. Therefore, we formally define the cost of type conversion as:
\begin{eqnarray}
    cost(elm_{old} \to elm_{new}) = cost_{del}(elm_{old}) + cost_{add}(elm_{new})
\end{eqnarray}
where $elm$ represents nodes and edges in a heterogeneous graph, and $elm_{old}, elm_{new}$ share identical structural attributes but differ in type. 

As shown in  Fig.~\ref{fig:hged}, the $\operatorname{HGED}(\boldsymbol{G_1}, \boldsymbol{G_2})$ is calculated as 5.

\begin{figure}[!htbp]
  \centering
  \includegraphics[width=\linewidth]{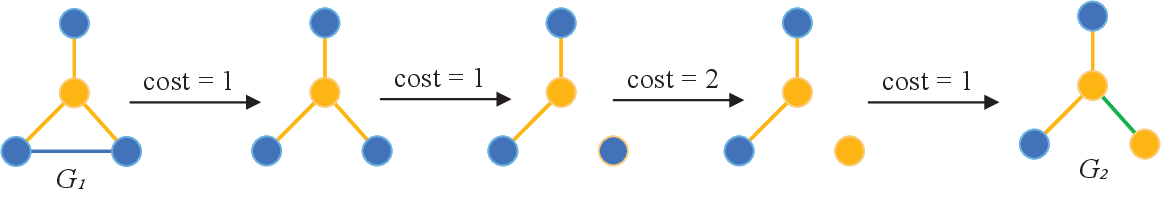}
  \caption{A toy example of HGED calculation, where different colors represent different types. The edit cost from the leftmost graph to the rightmost graph is 5.}
  \label{fig:hged}
  % \vspace{-10pt}
\end{figure}

\subsection{Heterogeneous Graph Similarity Learning}
The task of heterogeneous graph similarity learning (HGSL) aims to learn the similarity between heterogeneous graphs by matching their structural and semantic information. In this paper, HGSL is formulated as a HGED regression task.

Given two heterogeneous graphs $\boldsymbol{G_i}, \boldsymbol{G_j}$, and their normalized HGED value $s$, the objective of HGSL is to learn a similarity scoring function $\operatorname{f_{\theta}}$ that minimizes the loss:
\begin{eqnarray}\label{eq:e1050}
\mathcal{L} = \frac{1}{|\boldsymbol{T_r}|}\sum_{(i, j) \in \boldsymbol{T_r}} \operatorname{MSE} (\operatorname{f_{\theta}}(\boldsymbol{G_i},\boldsymbol{G_j}) \to s)
\end{eqnarray}
where $\boldsymbol{T_r}$ represents the set of trained graph pairs, $\theta$ is parameter of the model, $s \in \mathbbm{R}$ represent the ground-truth. In this paper, we model the function as an end-to-end neural learning paradigm, which uses GNNs as heterogeneous graph encoder and then quantify the discrepancy between $\boldsymbol{G_i}$ and $\boldsymbol{G_j}$.

\section{Method}
\subsection{Overview}
In this paper, Heterogeneous Graph Matching Network (HeGMN) is proposed to solve the HGSL problem. The overall framework is shown in Fig.~\ref{fig:model}, which consists of four main modules: (1) heterogeneous node encoding; (2) type-aligned graph matching; (3) type-aligned node matching; (4) graph similarity prediction. 

In the heterogeneous node encoding module, a siamese heterogeneous graph isomorphic network is proposed to learn the node representation for each heterogeneous graph pair.

In the type-aligned graph matching module, the sum pooling is first conducted on the node embedding matrix $\boldsymbol{Z_i}$ and $\boldsymbol{Z_j}$ respectively to obtain corresponding $|\boldsymbol{C}|$ type embeddings. Then, a multi-layer perceptron is applied to compute the matching score between type embeddings $\boldsymbol{T_i}$  and $\boldsymbol{T_j}$. Finally, an attention mechanism is used to fuse the matching results of different types.  

In the type-aligned node matching module, we propose to match the cross-graph nodes with the same type, which can not only eliminate the interference caused by matching nodes with different semantics but also reduce the computational complexity.

In the graph similarity prediction module, the graph-level and node-level matching features are combined and fed into fully connected layers. The overall network is trained in the HGED regression task with the mean square error (MSE) as the loss function.

Each module is detailed as follows.
\begin{figure*}
  \centering
  \includegraphics[width=\textwidth]{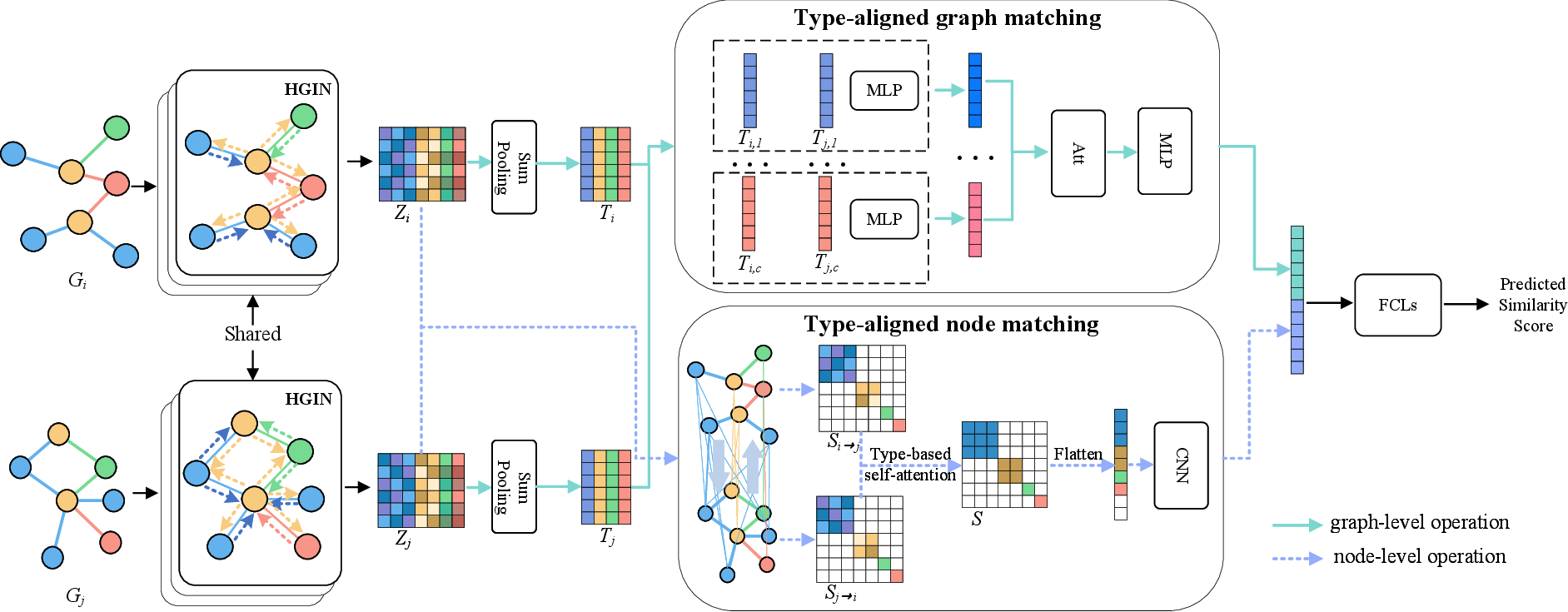}
  \caption{Overall framework of HeGMN. The framework contains four modules: heterogeneous node encoding, type-aligned graph matching, type-aligned node matching, and graph similarity prediction. The first module learns the node representation by siamese HGIN. The second module calculates the similarity of type embedding by MLP. The third module matches the cross-graph nodes with the same type. The fourth module combines the output vectors of the second and third modules and feeds them into FCL to predict the similarity. Nodes and edges in different colors represent different types in the heterogeneous graph.}
  \label{fig:model}
  % \vspace{-5pt}
\end{figure*}

\subsection{Heterogeneous Node Encoding - HGIN}
In the preliminary experiments (refer to the results on AIDS700nef in Fig.~\ref{fig:pre}), GIN \cite{xu2018powerful} shows superior performance compared to GCN and GAT in graph-level matching. Moreover, GIN requires fewer parameters and less computational complexity, which benefits from its injective nature when passing message. Formally, its message passing strategy is as follows:
\begin{eqnarray}\label{eq:e2}
\bm{z_n}^{(l)} = \operatorname{MLP}^{(l)}((1 + {\epsilon}^{(l)})\bm{z_n}^{(l-1)} + \sum_{m \in \boldsymbol{\mathcal{N}}(n)} \bm{z_m}^{(l - 1)})
\end{eqnarray}
where $\bm{z_n}^{(l)}$ denotes the embedding of node $n$ at the $l$-th layer, $\bm{z_n}^{(0)}$ is $\boldsymbol{X}$, $\epsilon$ is a learnable parameter and $\boldsymbol{\mathcal{N}}(n)$ is the neighbor set of node $n$.

However, the matching performance of GIN also deteriorates badly on the heterogeneous graph dataset ACM1000. We surmise that the diverse relationships between nodes in heterogeneous graphs should be encoded differently. Therefore, we propose the Heterogeneous Graph Isomorphism Network (HGIN), which additionally designs relational attention mechanism for different edge types to better aggregate semantic information (see Eq.~\ref{eq:e3} and Eq.~\ref{eq:e100}).
\begin{eqnarray}\label{eq:e3}
\bm{z_n}^{(l)} = \operatorname{MLP}^{(l)}((1 + \epsilon^{(l)})\bm{z_n}^{(l-1)} + \boldsymbol{\Gamma}) \\
\label{eq:e100}
\boldsymbol{\Gamma} = \sum_{r \in \boldsymbol{R}}\sum_{m \in \boldsymbol{\mathcal{N}^r}(n)} \frac{1}{k_{n, r}} \boldsymbol{W_r}^{(l-1)} \bm{z_m}^{(l - 1)}
\end{eqnarray}
where $r \in \boldsymbol{R}$ represents a relation (i.e., edge) type, $\boldsymbol{\mathcal{N}^r}(n)$ denotes the neighbor set of $n$ linked by $r, k_{n, r}$ is a normalized constant of $|\boldsymbol{\mathcal{N}^r}(n)|$, and $\boldsymbol{W_r}$ denotes the learnable attention weights shared by all relation types. Note that if the graph contains only one relation type, HGIN will degrade into GIN.

Different from RGCN \cite{schlichtkrull2018modeling}, which also uses relational attention for aggregation, HGIN employs an injective message passing strategy, enabled to pass the 1-WL test with less running time (see Theorem~\ref{thm:HGINequal1WL}). Moreover, HGIN is more favorable for GSL task since it focuses more on structural learning than RGCN.

\begin{theorem} \label{thm:GNN1WL}
Let $\operatorname{g_\theta}$ be a GNN with a sufficient number of layers.
If the following two conditions are met, $\operatorname{g_\theta}$ can map any graphs $\boldsymbol{G_i}$ and $\boldsymbol{G_j}$ to different embeddings when they are determined to be non-isomorphic by Weisfeiler-Lehman test:
\begin{enumerate}
    \item $\operatorname{g_\theta}$ aggregates and updates node features iteratively with
    \begin{eqnarray}\label{eq:e4}
    \bm{z_n}^{(l)} = \Upsilon(\bm{z_n}^{(l-1)}, \operatorname{g_\theta}({\bm{z_m}^{(l-1)}: m \in \boldsymbol{\mathcal{N}}(n)})
    \end{eqnarray}
    where $\Upsilon$ is injective.

    \item the graph-level readout of $\operatorname{g_\theta}$ is injective.
\end{enumerate}
\end{theorem}
This theorem has been proved in GIN \cite{xu2018powerful}. On this basis, we further prove that HGIN possesses equivalent expressive power to the Weisfeiler-Lehman test.

\begin{theorem}\label{thm:HGINequal1WL}
HGIN possesses equivalent expressive power to the Weisfeiler-Lehman test.
\end{theorem}

\begin{proof}\label{thm:proof}
According to Theorem~\ref{thm:GNN1WL}, both the message aggregation and the readout functions of HGIN need to be proven to satisfy the injective property.
\begin{enumerate}
    \item As shown in Eq.~\ref{eq:e2} and Eq.~\ref{eq:e3}, both HGIN and GIN use MLP as message aggregation function $\Upsilon$. Therefore, the first condition is satisfied.
    \item HGIN utilizes sum pooling function as its readout function, which satisfies the injective condition. Therefore, the second condition is satisfied. $\hfill\qedhere$
\end{enumerate} 
\end{proof}

However, the attention weights could significantly increase the computational time especially when there is a large number of relation types. Therefore, we further utilize matrix decomposition technique \cite{schlichtkrull2018modeling} to reduce computational complexity. Specifically, each attention weight matrix $\boldsymbol{W_r}$ is decomposed into a linear combination of shared parameter $\boldsymbol{\mathcal{P}_b}$ and relation-specific parameter $\bm{w_{(r,b)}}$, formally defined as:
\begin{eqnarray}\label{eq:e5}
\boldsymbol{W_r} = \sum^\mathcal{B}_{b = 1}\bm{w_{r, b}} \boldsymbol{{\mathcal{P}}}_b
\end{eqnarray}
where $\mathcal{B}$ is a hyperparameter indicating the number of subcategories to classify the relation type set $\boldsymbol{R}$, $\bm{w_{r, b}}$ represents the learnable parameters specific to relation type $r$, and $\boldsymbol{\mathcal{P}}_b$ is the shared parameter across all relation types. Here, $\bm{w_{r, b}}$ can learn the contribution degree of each relation type and $\boldsymbol{\mathcal{P}}_b$ can learn the intrinsic consistency between different relation types which are often not semantically independent. For instance, in social networks, \textit{follow} and \textit{remark} are two different relation types, yet both represent preferences for others. 

\subsection{Type-aligned Graph Matching}
In this module, the graph-level matching is conducted based on the alignment of node types.

Firstly, given the node embedding matrix $\boldsymbol{Z} \in \mathbbm{R}^{N \times D}$, where $D$ is the dimension of node, sum pooling is adopted for each node type $c$ to generate its type embedding:
\begin{eqnarray}\label{eq:e7}
\boldsymbol{T} = \sum_{c \in \boldsymbol{C}}\sum_{n \in \boldsymbol{V}} \bm{z_n} \cdot \mathbbm{I}(c_n = c)
\end{eqnarray}
where $\boldsymbol{C}$ is the set of node types, $\boldsymbol{T} \in \mathbbm{R}^{|\boldsymbol{C}| \times D}$ is the result of sum pooling, and $\mathbbm{I}(c_n = c)$ is an indicator function. If $c_n = c$, $\mathbbm{I}$ takes the value 1; otherwise, it takes the value 0.

Secondly, the cross-graph alignment is designed based on $\boldsymbol{T_i}$. Specifically, MLP is used to match the embeddings with the same type $c$ (e.g., $\bm{t_{i,c}}$ and $\bm{t_{j,c}}$ represent the embedding of type $c$ in graph $\boldsymbol{G_i}$ and $\boldsymbol{G_j}$, respectively), thus obtaining  type-aligned cross-graph alignment.
\begin{eqnarray}\label{eq:e8}
\bm{t_{c}} = \operatorname{MLP}([\bm{t_{i, c}}, \bm{t_{j, c}}])
\end{eqnarray}
where $\bm{t_{i, c}}$ is the embedding of type $c$ in $\boldsymbol{G_i}$, and $[\cdot]$ is the concatenation operation.

Thirdly, the attention pooling is applied to the type embeddings as different node types may have varying importance. Specifically, it is achieved by the global type context:
\begin{eqnarray}\label{eq:e9}
\bm{a} = \tanh(\frac{1}{|\boldsymbol{C}|} \boldsymbol{W_a} \sum_{c \in \boldsymbol{C}} \bm{t_c})
\end{eqnarray}
where $\boldsymbol{W_a}$ is a learnable weight matrix. Then, the sigmoid function is used to normalize the inner product of attention coefficients and type embeddings, ensuring that the node types with similar global context will have higher attention weights. The overall graph embedding $\bm{h}$ is the weighted sum of node type embeddings:
\begin{eqnarray}\label{eq:graphemb}
\bm{h} = \sum_{c \in \boldsymbol{C}}\sigma(\bm{t_c}^{\intercal} \ \bm{a})\bm{t_c} 
\end{eqnarray}
where $\sigma(\cdot)$ is the sigmoid function.

Finally, the graph-level matching vector $\bm{h}'$ is achieved using MLP:
\begin{eqnarray}\label{eq:gembedding}
\bm{h}' = \operatorname{MLP}(\bm{h}) \ (\bm{h}' \in \mathbbm{R}^\mathcal{D'})
\end{eqnarray}
where $\mathcal{D'}$ is the dimension of output.

\subsection{Type-aligned Node Matching}
Node-level matching is often used as a supplement to graph-level matching, which learns the correlation between node pairs across graphs. This paper proposes a novel node matching strategy based on node types.

Firstly, the interaction attention values between node embeddings in two graphs are computed. Here, a type mask technique is designed to conceal attention values between nodes with different types. Specifically, the nodes in one graph are treated as queries and the nodes in the other graph as keys and values. We use the output of HGIN to compute multi-head interaction attention weights, and simultaneously learn $\boldsymbol{S_{i \to j}}$ and $\boldsymbol{S_{j \to i}}$ to represent bidirectional similarity matrices between $\boldsymbol{G_i}$ and $\boldsymbol{G_j}$. Here, $\boldsymbol{S_{i \to j}}(n,m)$ represents the attention of node $n$ in $\boldsymbol{G_i}$ to node $m$ in $\boldsymbol{G_j}$, and vice versa.

\begin{eqnarray}\label{eq:e12}
\boldsymbol{S_{i \to j}}(n, m) = softmax(\frac{\boldsymbol{\mathcal{Q}_m\mathcal{K}_n}^{\intercal}}{\sqrt{d_k}}) \cdot \mathbbm{I}(\Phi(m) = \Phi(n))
\end{eqnarray}

\begin{eqnarray}\label{eq:e13}
\boldsymbol{S_{j \to i}}(n, m) = softmax(\frac{\boldsymbol{\mathcal{Q}_n\mathcal{K}_m}^{\intercal}}{\sqrt{d_k}}) \cdot \mathbbm{I}(\Phi(n) = \Phi(m))
\end{eqnarray}
where $\boldsymbol{\mathcal{Q} = ZW_\mathcal{Q}}$ and $\boldsymbol{\mathcal{K} = ZW_\mathcal{K}}$. It is worth noting that only the attention weights between nodes with the same type are computed  with the help of type mask mechanism. 

Secondly, these similarity matrices are aligned to be compared in the same representation space with a self-attention mechanism. Specifically, the learned multiple head interaction similarity matrices are flattened to obtain a similarity vector $\bm{s}$, which is further normalized to facilitate smooth gradient propagation. The similarity matrix is updated as:

\begin{eqnarray}\label{eq:simtrix}
\boldsymbol{S} = \operatorname{softmax}(\frac{\boldsymbol{\mathcal{Q}} \cdot \boldsymbol{\mathcal{K}}^{\intercal}} {\sqrt{d_k}}) \cdot \boldsymbol{\mathcal{V}} \cdot \mathbbm{I}(\Phi(m) = \Phi(n))
\end{eqnarray}
where $d_k$ is a scaling factor, the query $\boldsymbol{\mathcal{Q}=S \cdot W_\mathcal{Q}}$, the key $\boldsymbol{\mathcal{K}=S \cdot W_\mathcal{K}}$, and the value $\boldsymbol{V=S \cdot W_V}$. Through the indicator function, the cross-graph attention is only calculated between the nodes of the same type, improving the running efficiency.

Finally, after alignment, the node-level similarity matrices are flattened and fed into a CNN to learn the overall features of node-level matching. We use the CNN based on cross filter \cite{tan2023exploring}, and the output is denoted as $\bm{s}' \in \mathbbm{R}^\mathcal{D}$, where $\mathcal{D}$ is a hyperparameter, representing the output dimension.

\subsection{Training On Graph Similarity Prediction}
In the last module, the graph-level matching vector $\bm{h}' \in \mathbbm{R}^\mathcal{D'}$ and the node-level matching vector $\bm{s}'\in \mathbbm{R}^\mathcal{D}$ are concatenated and fed into the FCLs to predict the similarity $\hat{s}_{i,j}$ of $\boldsymbol{G_i}$ and $\boldsymbol{G_j}$. We use MSE as the final loss function, which is defined as follows:

\begin{eqnarray}\label{eq:e15}
\mathcal{L} = \frac{1}{|\boldsymbol{T_r}|}\sum_{(i, j) \in \boldsymbol{T_r}}(\hat{s}_{i, j} - s_{i, j})^2
\end{eqnarray}
where $\boldsymbol{T_r}$ represents the set of trained graph pairs, and $s_{i,j}$ represents the ground-truth similarity of $\boldsymbol{G_i}$ and $\boldsymbol{G_j}$, i.e., their HGED value.

\section{Experiments}
In this section, we empirically evaluate the performance of the HeGMN method compared with recently proposed baselines on HGED regression task. Our code and data are available at https://github.com/alvinsang1906/HeGMN.

\subsection{Datasets}
Due to the large scale and high complexity of the existing heterogeneous graph datasets, we conduct random subgraphs sampling from three real heterogeneous graphs (i.e., ACM, DBLP, IMDB) using the Breadth-First Search algorithm (BFS) and then construct pairs of heterogeneous graphs. The number of nodes in the sampled subgraph is set to no more than 16 in order to obtain accurate HGED values \cite{blumenthal2020exact}. Moreover, the sampled graph is expected to contain abundant types to avoid degrading into a homogeneous graph. As a result, we not only limit the size of the sampled subgraphs, but also incorporate the elements with more diverse types present in the original graph. In addition to three sampled datasets mentioned above, the real-world compound dataset MUTAG is also utilized because the different kinds of atoms can be regarded as heterogeneous nodes. The statistics for all datasets are provided in Table~\ref{tab:statics}.

\begin{table}[!htbp]\small
    \centering
  \caption{Detail of datasets.}
  \label{tab:statics}
  \tabcolsep=0.018\linewidth
      \begin{tabular*}{\linewidth}{cccccc}
        \toprule
            \textbf{Datasets} & $\boldsymbol{|G|}$ & \textbf{Avg.}$\boldsymbol{|V|}$ & \textbf{Avg.}$\boldsymbol{|E|}$ & 
            \textbf{Types of} $\boldsymbol{V}$ & \textbf{Types of} $\boldsymbol{E}$ \\ 
        \midrule
            \textbf{ACM1000} & 1000 & 7.25 & 8.15 & 4 & 4  \\ 
            \textbf{DBLP700} & 700 & 9.04 & 8.42 & 3 & 3  \\ 
            \textbf{IMDB1200} & 1200 & 11.03 & 10.19 & 3 & 3  \\ 
            \textbf{MUTAG} & 188 & 17.93 & 19.79 & 7 & 4  \\ 
        \bottomrule
      \end{tabular*}
\end{table}

For each graph pair $(\boldsymbol{G_i}, \boldsymbol{G_j})$ in the sampling pool, we calculate their normalized HGED score as ground truths.
\begin{eqnarray}\label{eq:e16}
\operatorname{HGED}_{norm}(\boldsymbol{G_i}, \boldsymbol{G_j}) = exp({- \frac{\operatorname{HGED}(\boldsymbol{G_i}, \boldsymbol{G_j})}{(N + M) / 2}}) \in (0, 1]
\end{eqnarray}
where $N$ and $M$ represent the number of nodes in the two graphs, respectively. The larger the $\operatorname{HGED}_{norm}$ value is, the more structurally similar the two graphs are.

\subsection{Implementation Settings}
\label{subsection1}
All experiments are conducted on an Ubuntu 23.04 system equipped with two NVIDIA GTX 4090 GPUs and one AMD Ryzen 9 7950X CPU. For each experiment, we randomly divided the graph  into training, testing, and validation sets with proportions of 6:2:2, respectively. The training samples are paired graphs from the training set, while the validation and testing samples are paired graphs one of which is from the validation and testing sets, respectively, and the other of which is from the training set.  The batch size of the model is set to 128 and the model is trained for 10,000 epochs. The learning rate is set to 0.001 and AdamW is used as the experimental optimizer. Validation starts from the 100th epoch, and we employed an early stopping strategy, i.e., terminating training if there is no decrease in validation loss for consecutive 100 epochs, to select the final model based on the minimum validation loss.

In HeGMN model, a three-layer HGIN is used to learn node representations and set the output dimension of both graph-level matching vectors and node-level matching vectors to 128. Finally, we use four layers FCLs to output the predicted similarity results, with dimensions set to 128, 64, 32 and 1, respectively. 

\subsection{Comparison Methods}
We compared HeGMN against its competitors, including 4 base GNNs (GCN \cite{kipf2016semi}, GIN \cite{xu2018powerful}, GAT \cite{velivckovic2018graph}, RGCN \cite{schlichtkrull2018modeling}) using only graph-level matching and 8 advanced GSL methods (SimGNN \cite{bai2019simgnn}, GMN \cite{li2019graph}, GraphSim \cite{bai2020learning}, MGMN \cite{ling2021multilevel}, ERIC \cite{zhuo2022efficient}, TaGSim \cite{bai2021tagsim}, NA-GSL \cite{tan2023exploring}, GraSP \cite{zheng2024grasp}) using both graph-level and node-level matching. Given that existing GSL methods perform on homogeneous graphs, for fair comparison, we follow the heterogeneous data processing guidelines outlined in HGB \cite{lv2021we}, that is, heterogeneous graph datasets are preprocessed to obtain unified feature representations before feeding into respective models.

Experiments are conducted using the official settings for the above methods. For GMN, which is trained on classification task not on regression task, we replaced its last layer with FCLs. 

\subsection{Evaluation Metrics}
Our metric selection aligns with established practices in the GSL task. Recent state-of-the-art works \cite{bai2019simgnn, tan2023exploring} adopt $MSE$ for numerical accuracy and Spearman/Kendall's correlation coefficient \cite{spearman1904proof, kendall1938new} for ranking consistency. This convention stems from two key considerations: 
\begin{itemize}
    \item $MSE$, the mean squared error between the predicted similarity scores and the ground truths, directly measures regression error magnitude critical for downstream applications like edit cost estimation.
    \item Ranking correlation metrics (Spearman/Kendall) are standard for evaluating retrieval systems where ordering preservation is essential. While these metrics do not explicitly capture structural edit path fidelity, their widespread adoption ensures fair comparison with prior art.
\end{itemize}

The following metrics are also used as the complement which evaluate ranking results: $p@k$, which evaluates the intersection of the predicted top $k$ results and the true top $k$ results; $Time$, on the basis of analyzing the time complexity, we further verified the accuracy of the analysis by recording the time required to compute the similarity of the same graph pairs.

\subsection{Results}
\paragraph{Heterogeneous graph similarity regression}
\begin{table}[p]\large
\centering
\setlength{\tabcolsep}{2pt}
\caption{Comparison of HGED approximations on ACM1000 and DBLP700 for MSE, $\rho$, $\tau$, $p@10$, $p@20$. The unit of MSE value is $10^{-3}$. For each column, the best results are in bold and the second best results are underlined. The last row represents the difference in performance between HeGMN and the remaining optimal method.}
\resizebox{\textwidth}{!}{
\begin{tabular}{c ccccc ccccc}
\toprule
    & \multicolumn{5}{c}{ACM1000} 
    & \multicolumn{5}{c}{DBLP700} \\
    & $MSE \downarrow$ & $\rho \uparrow$ & $\tau \uparrow$ & $p@10 \uparrow$ & $p@20 \uparrow$
    & $MSE \downarrow$ & $\rho \uparrow$ & $\tau \uparrow$ & $p@10 \uparrow$ & $p@20 \uparrow$ \\
\cmidrule(r){2-6} 
\cmidrule(r){7-11} 
    GCN-MEAN & 6.772 & 0.950  & 0.851 & 0.391 & 0.498 & 2.516 & 0.962  & 0.882 & 0.757 & 0.769 \\ 
    GIN-MEAN & 6.156 & 0.924  & 0.808 & 0.318 & 0.419 & 2.307 & 0.959  & 0.881 & 0.775 & 0.801 \\ 
    GAT-MEAN & 6.097 & 0.922  & 0.799 & 0.314 & 0.399 & 4.012 & 0.940  & 0.841 & 0.629 & 0.650 \\ 
    RGCN-MEAN & 1.408 & 0.983  & 0.914 & 0.662 & 0.741 & 0.895 & 0.984  & 0.922 & 0.865 & 0.870 \\
\midrule[0.2pt]
    SimGNN & 4.564 & 0.945  & 0.835 & 0.380 & 0.463 & 2.934 & 0.953  & 0.854 & 0.705 & 0.730 \\ 
    GMN & 5.435 & 0.934  & 0.809 & 0.288 & 0.383 & 1.541 & 0.972  & 0.895 & 0.832 & 0.845 \\ 
    GraphSim & 4.229 & 0.939  & 0.823 & 0.436 & 0.530 & 3.149 & 0.948  & 0.840 & 0.653 & 0.690 \\ 
    MGMN & 5.023 & 0.904 & 0.841 & 0.521 & 0.558 & 1.631 & 0.937 & 0.880 & 0.759 & 0.764 \\ 
    ERIC & 3.731 & 0.901 & 0.836 & 0.484 & 0.536 & 1.309 & 0.955 & 0.891 & 0.776 & 0.778 \\ 
    TaGSim & 3.633 & 0.910 & 0.827 & 0.534 & 0.608 & 1.398 & 0.961 & 0.887 & 0.743 & 0.746 \\
    NA-GSL & 3.440 & 0.956 & 0.848 & 0.579 & 0.637 & 1.469 & 0.948 & 0.884 & 0.762 & 0.776 \\
    GraSP & 4.101 & 0.917 & 0.832 & 0.479 & 0.534 & 1.477 & 0.951 & 0.883 & 0.756 & 0.771 \\
\midrule[0.2pt]
    HGIN-MEAN & $\underline{1.039}$ & $\underline{0.987}$ & $\underline{0.926}$ & $\underline{0.773}$ & $\underline{0.847}$ & $\underline{0.881}$ & $\underline{0.985}$ & $\underline{0.928}$ & $\underline{0.874}$ & $\underline{0.883}$ \\ 
    \makecell[c]{HeGMN \\ ~} & \makecell[c]{$\textbf{0.792}$ \\ $\downarrow 43.75\%$} & \makecell[c]{$\textbf{0.989}$ \\ $\uparrow0.61\%$}  & \makecell[c]{$\textbf{0.934}$ \\ $\uparrow2.19\%$} & \makecell[c]{$\textbf{0.797}$ \\ $\uparrow20.39\%$} & \makecell[c]{$\textbf{0.848}$ \\ $\uparrow14.44\%$} & \makecell[c]{\textbf{0.712} \\ $\downarrow20.45\%$} & \makecell[c]{$\textbf{0.986}$ \\ $\uparrow0.20\%$}  & \makecell[c]{$\textbf{0.931}$ \\ $\uparrow0.98\%$} & \makecell[c]{$\textbf{0.891}$ \\ $\uparrow3.01\%$} & \makecell[c]{$\textbf{0.888}$ \\ $\uparrow2.07\%$} \\ 
\bottomrule
\end{tabular}
}
\label{tab: t2}
\end{table}

\begin{table}[p]\large
\centering
\setlength{\tabcolsep}{2pt}
\caption{Comparison of HGED approximations on IMDB1200 and MUTAG for MSE, $p@10$, $p@20$.}
\resizebox{\textwidth}{!}{
\begin{tabular}{c ccccc ccccc}
\toprule
    & \multicolumn{5}{c}{IMDB1200} 
    & \multicolumn{5}{c}{MUTAG} \\
    & $MSE \downarrow$ & $\rho \uparrow$ & $\tau \uparrow$ & $p@10 \uparrow$ & $p@20 \uparrow$
    & $MSE \downarrow$ & $\rho \uparrow$ & $\tau \uparrow$ & $p@10 \uparrow$ & $p@20 \uparrow$ \\
\cmidrule(r){2-6} 
\cmidrule(r){7-11} 
    GCN-MEAN & 10.828 & 0.790  & 0.633 & 0.095 & 0.144 & 5.933 & 0.922  & 0.792 & 0.555 & 0.670 \\ 
    GIN-MEAN & 11.670 & 0.785  & 0.629 & 0.093 & 0.150 & 5.208 & 0.940  & 0.823 & 0.629 & 0.689 \\ 
    GAT-MEAN & 11.047 & 0.788  & 0.630 & 0.084 & 0.130 & 8.863 & 0.920  & 0.795 & 0.547 & 0.636 \\ 
    RGCN-MEAN & 2.547 & 0.947  & 0.837 & 0.473 & 0.573 & $\underline{3.398}$ & $\underline{0.953}$ & $\underline{0.853}$ & $\underline{0.704}$ & $\underline{0.757}$\\
\midrule[0.2pt]
    SimGNN & 10.523 & 0.788  & 0.632 & 0.091 & 0.133 & 4.761 & 0.935  & 0.811 & 0.647 & 0.724 \\ 
    GMN & 10.183 & 0.809  & 0.654 & 0.122 & 0.178 & 4.393 & 0.950  & 0.837 & 0.668 & 0.745 \\ 
    GraphSim & 11.270 & 0.769  & 0.609 & 0.086 & 0.128 & 5.016 & 0.918 & 0.789 & 0.563 & 0.678 \\ 
    MGMN & 9.768 & 0.830 & 0.721 & 0.396 & 0.429 & 4.272 & 0.911 & 0.790  & 0.625 & 0.673 \\ 
    ERIC & 8.448 & 0.856 & 0.741 & 0.398 & 0.434 & 4.064 & 0.906 & 0.774 & 0.605 & 0.659 \\ 
    TaGSim & 7.341 & 0.868 & 0.750 & 0.376 & 0.413 & 4.135 & 0.897 & 0.784 & 0.627 & 0.660 \\
    NA-GSL & 7.581 & 0.879 & 0.760 & 0.412 & 0.455 & 3.998 & 0.937 & 0.813 & 0.641 & 0.702 \\
    GraSP & 7.699 & 0.849 & 0.733 & 0.383 & 0.441 & 4.532 & 0.892 & 0.753 & 0.612 & 0.631 \\
\midrule[0.2pt]
    HGIN-MEAN & $\underline{2.372}$ & $\underline{0.953}$ & $\underline{0.846}$ & $\underline{0.504}$ & $\underline{0.584}$ & 3.818 & 0.945  & 0.833 & 0.653 & 0.731 \\ 
    \makecell[c]{HeGMN \\ ~} & \makecell[c]{$\textbf{2.070}$ \\ $\downarrow18.73\%$} &\makecell[c]{$\textbf{0.954}$ \\ $\uparrow0.74\%$}  & \makecell[c]{$\textbf{0.849}$ \\ $\uparrow1.43\%$}  & \makecell[c]{$\textbf{0.526}$ \\ $\uparrow11.21\%$} & \makecell[c]{$\textbf{0.526}$ \\ $\uparrow15.36\%$} & \makecell[c]{$\textbf{0.661}$ \\ $\downarrow 17.92\%$} & \makecell[c]{$\textbf{0.958}$ \\ $\uparrow0.52\%$}  & \makecell[c]{$\textbf{0.858}$ \\ $\uparrow0.59\%$} & \makecell[c]{$\textbf{0.717}$ \\ $\uparrow1.85\%$} & \makecell[c]{$\textbf{0.782}$ \\ $\uparrow3.30\%$} \\ 
\bottomrule
\end{tabular}
}
\label{tab: t3}
\end{table}

The experimental results of all methods are summarized in Table~\ref{tab: t2} and Table~\ref{tab: t3}. HGIN-MEAN is the variant of HeGMN which uses three layers of HGIN as the encoder and mean pooling on node representations to obtain graph-level representation. 

As shown in Table~\ref{tab: t2} and Table~\ref{tab: t3}, the GSL methods equipped with node-level matching generally outperform those with only graph-level matching, consistent with previous observations on homogeneous GSL experiments.

HGIN-MEAN significantly outperforms existing base GNNs methods (GCN-MEAN, GIN-MEAN and GAT-MEAN) which use homogeneous graph encoder. For example, HGIN-MEAN shows an improvement of 82.26\% compared to GIN-MEAN. It verifies the significance of encoding type information in representation for HGSL task. HGIN-MEAN outperforms RGCN-MEAN on most datasets (i.e., ACM1000, DBLP700 and IMDB1200), probably because the injective function in HGIN is endowed with better representation ability. However, on the MUTAG dataset, HGIN's performance slightly lags behind RGCN. Considering the learnable parameters of RGCN are determined by the number of neighbors, unlike HGIN, which is determined by the number of neighbor types, RGCN could capture more information on larger graphs.

Overall, HeGMN consistently outperforms all its competitors, achieving highest performance across all four datasets. For example, HeGMN improves HGIN-MEAN by adding an efficient type-aligned node matching branch. HeGMN significantly surpasses RGCN-MEAN  by 60.9\%, which soundly verifies the effectiveness of node-level matching. It is also observed that the larger the scale of the dataset, the more significant the improvement of HeGMN over HGIN-MEAN. It is because the structure and semantics of a larger heterogeneous graph are more complex which can be better captured by the type-aligned node matching module in HeGMN.

\paragraph{Ablation study}
To further study the effectiveness of key modules in HeGMN, we conducted a series of ablation experiments. All variants and their results are shown in Table~\ref{tab: t4} and Table~\ref{tab: t5}. HeGMN-w GIN replaces HGIN in HeGMN with plain GIN for encoding. HeGMN-w/o node-match removes the node matching module in HeGMN. HeGMN-w/o type-mask uses node-level matching but computes interactions between all inter-graph nodes without masking irrelevant types.

\begin{table}[!]\large
\centering
\setlength{\tabcolsep}{2pt}
\caption{Ablation experiment results on ACM1000 and DBLP700 for MSE, $\rho$, $\tau$, $p@10$, $p@20$.}
\resizebox{\textwidth}{!}{
\begin{tabular}{c ccccc ccccc}
\toprule
    & \multicolumn{5}{c}{ACM1000} 
    & \multicolumn{5}{c}{DBLP700} \\
    & $MSE \downarrow$ & $\rho \uparrow$ & $\tau \uparrow$ & $p@10 \uparrow$ & $p@20 \uparrow$
    & $MSE \downarrow$ & $\rho \uparrow$ & $\tau \uparrow$ & $p@10 \uparrow$ & $p@20 \uparrow$ \\
\cmidrule(r){2-6} 
\cmidrule(r){7-11} 
    w GIN & 5.161 & 0.920 & 0.801 & 0.482 & 0.609 & 2.007 & 0.960 & 0.893 & 0.795 & 0.832 \\
    w/o node-match & 0.980 & 0.967 & 0.930 & 0.789 & 0.844 & 0.853 & 0.983 & 0.928 & 0.885 & 0.880\\
    w/o type-mask & 0.986 & 0.979 & 0.911 & 0.654 & 0.801 & 1.241 & 0.950 & 0.923 & 0.825 & 0.836\\
\midrule[0.2pt]
    HeGMN & $\textbf{0.792}$ & $\textbf{0.989}$ & $\textbf{0.934}$ & \textbf{0.797}& \textbf{0.848}&  $\textbf{0.712}$ & $\textbf{0.986}$ & $\textbf{0.931}$ & \textbf{0.891}& \textbf{0.888}\\
\bottomrule
\end{tabular}
}
\label{tab: t4}
\end{table}

\begin{table}[!]\large
\centering
\setlength{\tabcolsep}{2pt}
\caption{Ablation experiment results on IMDB1200 and MUTAG for MSE, $\rho$, $\tau$, $p@10$, $p@20$.}
\resizebox{\textwidth}{!}{
\begin{tabular}{c ccccc ccccc}
\toprule
    & \multicolumn{5}{c}{IMDB1200} 
    & \multicolumn{5}{c}{MUTAG} \\
    & $MSE \downarrow$ & $\rho \uparrow$ & $\tau \uparrow$ & $p@10 \uparrow$ & $p@20 \uparrow$
    & $MSE \downarrow$ & $\rho \uparrow$ & $\tau \uparrow$ & $p@10 \uparrow$ & $p@20 \uparrow$ \\
\cmidrule(r){2-6} 
\cmidrule(r){7-11} 
    w GIN & 10.773 & 0.789 & 0.631& 0.193&  0.251& 4.801 & 0.933 & 0.880&  0.649 & 0.709\\
    w/o node-match & 2.213 & 0.949 &  0.847& 0.517 & 0.592 & 3.503 & 0.943 & 0.847& 0.679& 0.744 \\
    w/o type-mask &  2.427&  0.935& 0.796 & 0.410& 0.452& 4.091 & 0.914 & 0.808& 0.618& 0.701 \\
\midrule[0.2pt]
    HeGMN & $\textbf{2.070}$ & $\textbf{0.954}$ & $\textbf{0.849}$ & $\textbf{0.526}$ & $\textbf{0.661}$ & $\textbf{2.789}$ & $\textbf{0.958}$ & $\textbf{0.858}$ & $\textbf{0.717}$ & $\textbf{0.782}$ \\
\bottomrule
\end{tabular}
}
\label{tab: t5}
\end{table}

As shown in Table~\ref{tab: t4} and Table~\ref{tab: t5}, substituting HGIN with GIN results in a performance decrease, indicating the need to incorporate type information when representing heterogeneous nodes. Removing the node-level matching module (i.e., HeGMN-w/o node-match) leads to varying degrees of performance degradation across four datasets. The comparison between HeGMN and HeGMN-w/o type-mask further illustrates that masking nodes with different types is effective, probably because matching between nodes with different types will cause noise. Moreover, the type mask mechanism can highly reduce the time consumption of the model, which will be analyzed in the next section.

\begin{figure}[!htbp]
    \centering
    \subfigure[Number of Layers]{
        \label{fig: fig5_1}
        \includegraphics[width=0.45\textwidth]{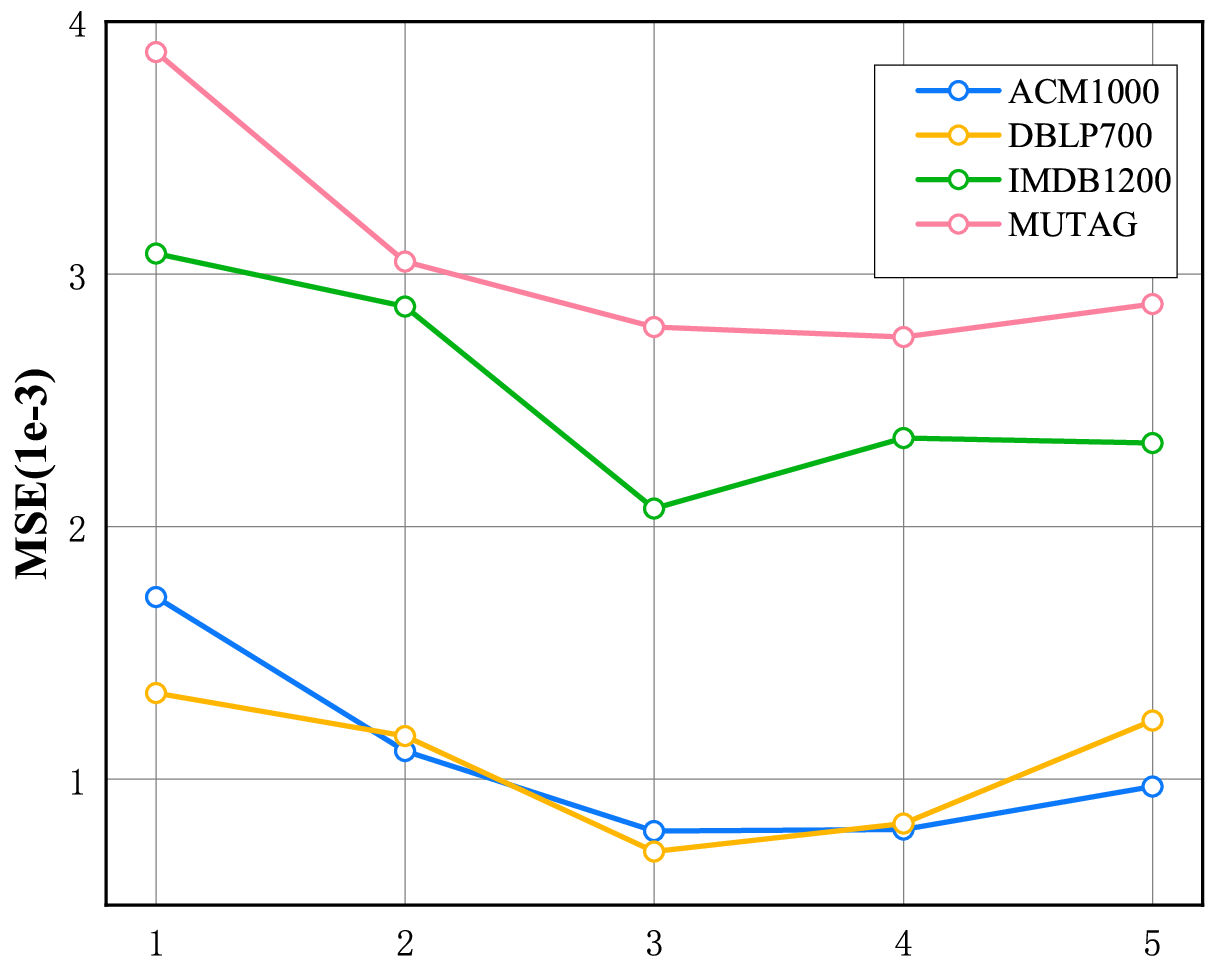}
    }
    \subfigure[$\mathcal{B}$]{
        \label{fig: fig5_2}
        \includegraphics[width=0.45\textwidth]{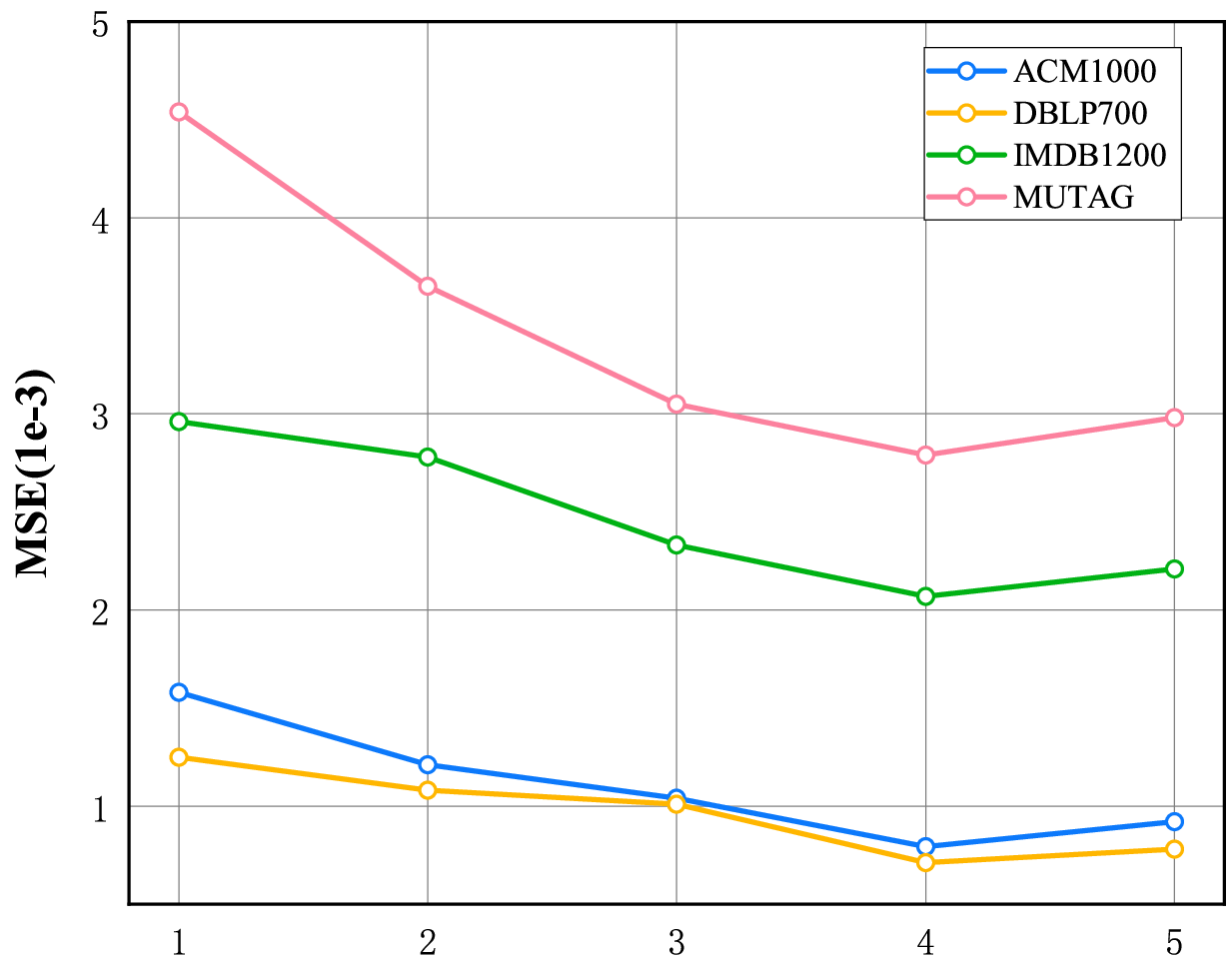}
    }
    \caption{(a) The MSE results when the number of HGIN layers is $[1, 2, 3, 4, 5]$. (b) The MSE results when the hyperparameter $\mathcal{B}$ in matrix decomposition is $[1, 2, 3, 4, 5]$.}
    \label{fig: fig5}
\end{figure}

\begin{figure*}[h]
    \centering
    \includegraphics[width=\textwidth]{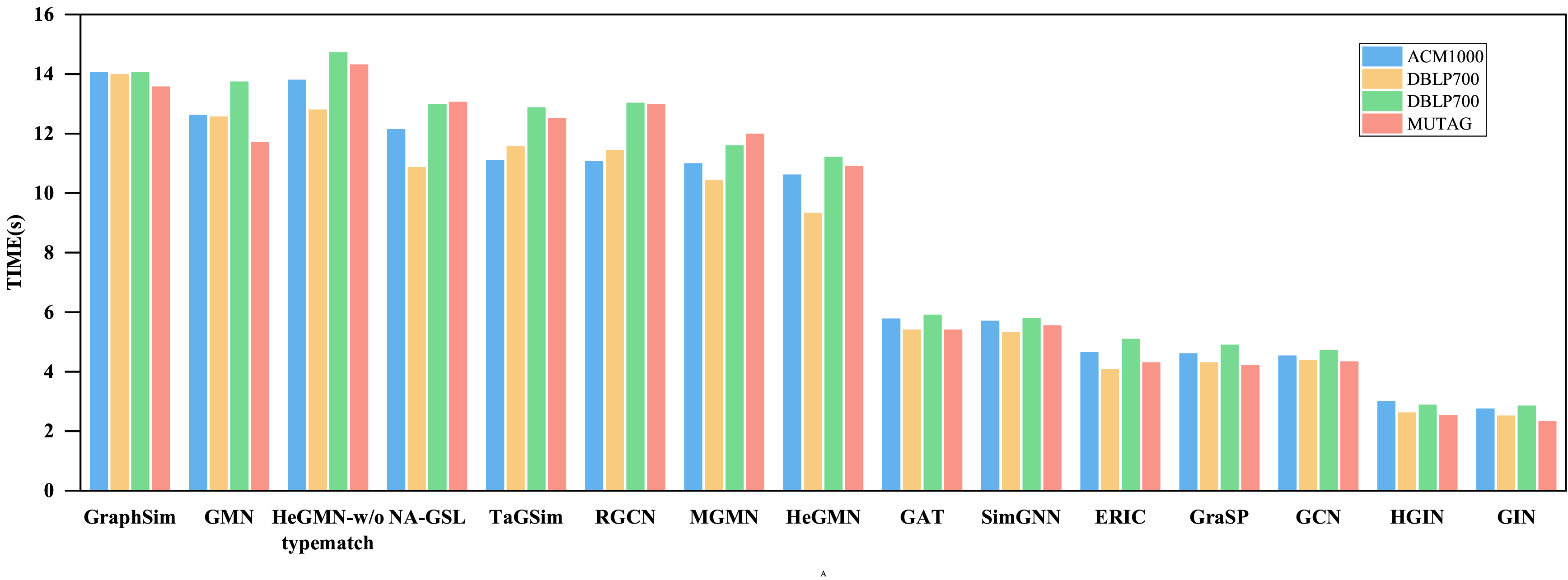}
    \caption{Comparison of running time (in seconds) to calculate the similarity of the same 1280 graph pairs.}
    \label{fig:figure4}
\end{figure*}

\paragraph{Time complexity analysis and efficiency in running time}
Running time is another important evaluation aspect of deep GSL methods, since it is the main advantage over the traditional GSC methods. For all methods, we computed the similarity on the same 1280 graph pairs and recorded the time consumption in Fig.~\ref{fig:figure4}.

\begin{itemize}
    \item \textbf{Heterogeneous node encoding.} A siamese HGIN is employed in the node encoding module to embed nodes from two input heterogeneous graphs. For each graph, the time complexity is dominated by:
    \begin{eqnarray}
        O(L \cdot (|E| \cdot d + |V|^2 \cdot d) )
    \end{eqnarray}
    where $L$ is the number of HGIN layers, $|E|, |V|$ are the number of edges and nodes in a heterogeneous graph, respectively, $d$ is the dimension of feature.
    \item \textbf{Type-aligned graph matching.} Sum pooling aggregates node embeddings by type for each graph. The complexity of sum pooling is: $O(2 \cdot |V| \cdot d)$, where $|V|$ is the total nodes across all types in one graph. For each node type $t \in T$, an MLP matches type-specific representations from two graphs. Assuming a single-layer MLP with input dimension $2d$ and output dimension $k$, the time complexity is: $O(T \cdot (2d \cdot k))$. Overall, the total time complexity of graph matching is dominated by:
    \begin{eqnarray}
        O(2 \cdot |V| \cdot d + T \cdot 2d \cdot k)
    \end{eqnarray}
    \item \textbf{Type-aligned node matching.} Computing pairwise similarities between nodes of type $t$ in the two heterogeneous graphs, the complexity is:
    \begin{eqnarray}
        O(\sum_{t \in T} (|V_{1, t}| \cdot |V_{2, t}| \cdot d))
    \end{eqnarray}
    where $|V_{1, t}|, |V_{2, t}|$ are the number of type-$t$ nodes in two graph $G_1, G_2$, respectively.
\end{itemize}
Generally, the GSL methods with node-level matching, despite the excellent performance, exhibit slower computation speeds than those with only graph-level matching. It is usually time-consuming to compute interactions between all pairs of nodes across graphs. Although SimGNN consumes less time owing to the use of histograms for processing, its performance improvement is limited since the histograms are non-differentiable and cannot be learned. Meanwhile, the running efficiency of HGIN is slightly lower than that of GIN as the attention weights set on different types of neighbors involve additional matrix multiplication calculations. However, HGIN still outperforms other encoders in terms of processing speed, indicating the effectiveness of matrix decomposition module. Among all GSL models with node-level matching, HeGMN consumes the least time, verifying the type mask strategy can significantly improve the efficiency of heterogeneous graph matching. Meanwhile, it also provides some insights into the optimization of homogeneous graph matching, particularly in node grouping interaction.

\paragraph{Parameter sensitivity}
HeGMN takes two hyperparameters, i.e., the number of HGIN layers and the parameter $\mathcal{B}$, which represents the reclassification of relation types in matrix decomposition. The results of parameter sensitivity analysis (see Fig.~\ref{fig: fig5}) show that HeGMN remains relatively stable as parameters vary. It is worth noting that $\mathcal{B}$ has a greater impact on model performance on the MUTAG dataset. The possible reason is that the MUTAG dataset covers 7 node types, and when $\mathcal{B}$ is small, multiple node types are treated as a single category, resulting in information loss. The minimum MSE is consistently achieved across four datasets when the number of HGIN layers is 3 and $\mathcal{B}$ is 4, which are therefore the default settings of HeGMN.

\begin{figure}[htp]
\centering
    \subfigure[ACM1000]{
        \includegraphics[width=0.6\linewidth, height=3.9cm]{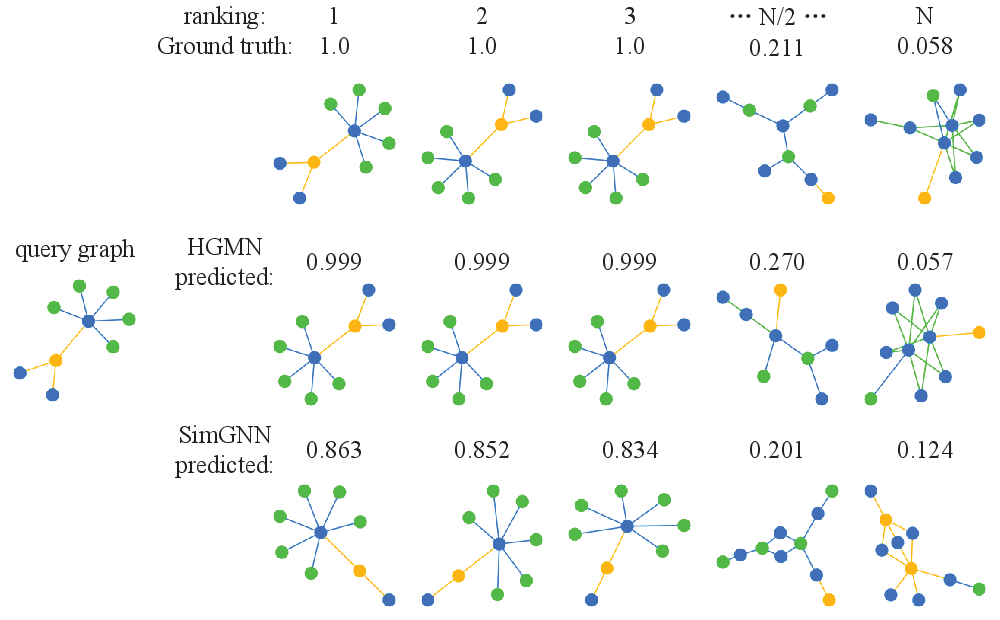}
        \label{fig:fig6a}
    }
    \subfigure[DBLP700]{
        \includegraphics[width=0.6\linewidth, height=3.9cm]{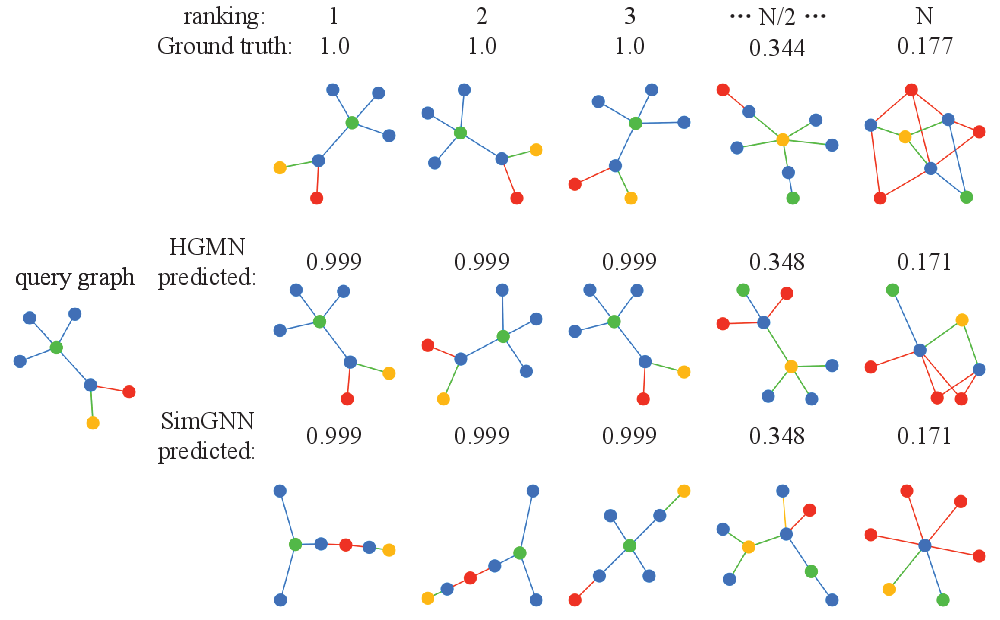}
        \label{fig:fig6b}
    }
    \subfigure[IMDB1200]{
        \includegraphics[width=0.6\linewidth, height=3.9cm]{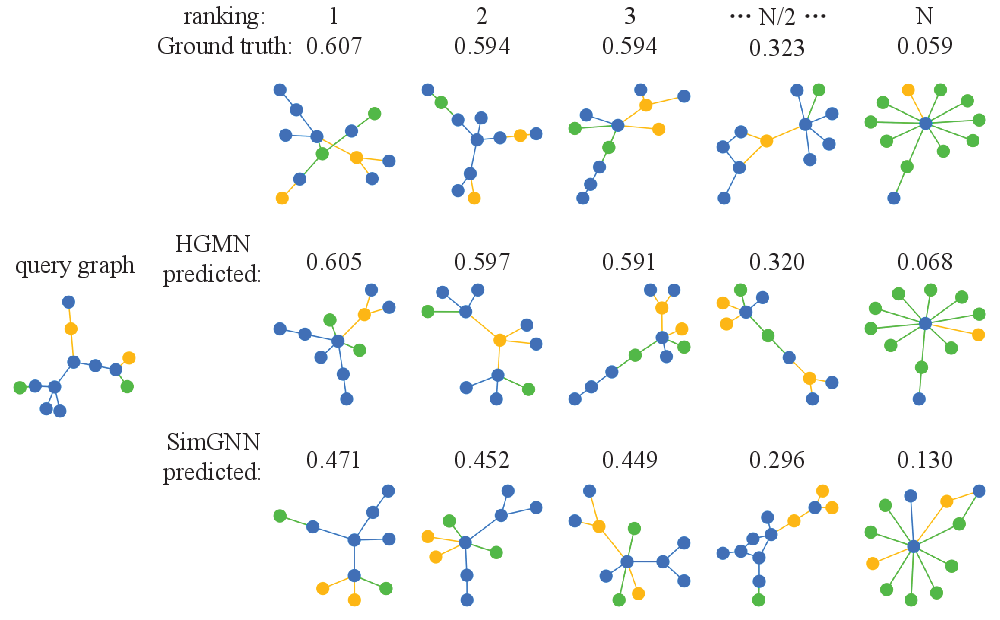}
        \label{fig:fig6c}
    }
    \subfigure[MUTAG]{
        \includegraphics[width=0.6\linewidth, height=3.9cm]{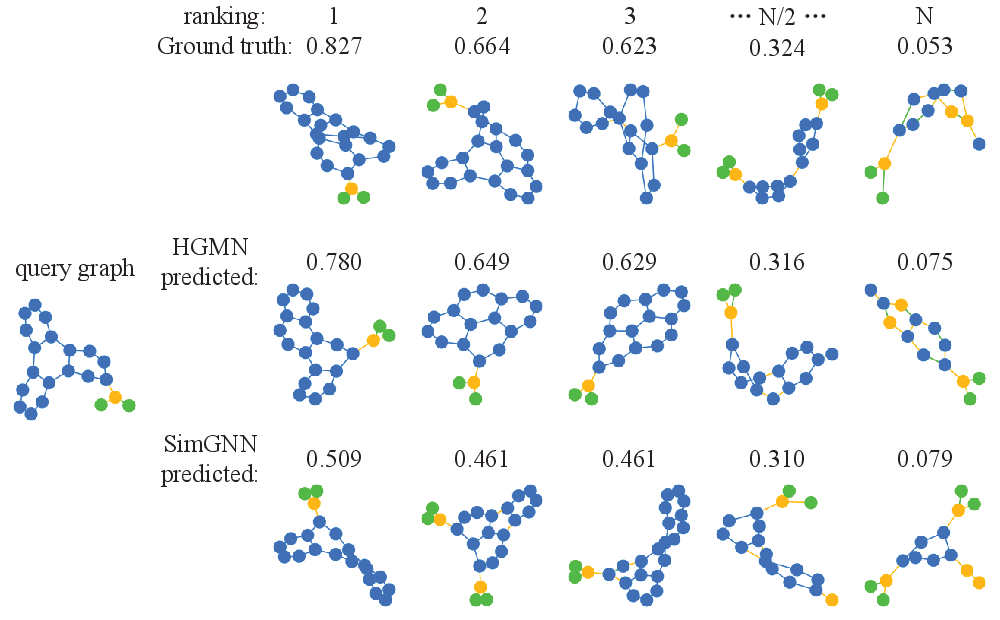}
        \label{fig:fig6d}
    }
    \caption{Case study of graph retrieval.}
    \label{fig:figcase}
\end{figure}

\subsection{Case studies}
To further investigate the retrieval capability of HeGMN, we conducted graph query experiments on ACM1000, DBLP700, IMDB1200, and MUTAG datasets, as depicted in Fig.~\ref{fig:figcase}. In each case, the leftmost graph represents the query graph taken from the training dataset and the right graphs are the retrieved graphs. Similar experiments were conducted on SimGNN for comparison. The retrieval result shows that HeGMN can retrieve graphs that are relatively more similar to the query graph, particularly on smaller-scale graph datasets like ACM1000 and DBLP700. For example, the top-3 query results are isomorphic to the query graph. For larger-scale graphs such as IMDB1200 and MUTAG, HeGMN is also able to retrieve sufficiently approximate results. It is worth noting that the task of graph searching can be addressed using graph matching but we did not compare with index-based graph searching methods as they essentially serve different purposes. HeGMN aims to learn accurately the similarity between pairs of graphs rather than search efficiently similar graphs in a graph database.

\section{Conclusion}
Current graph matching methods based on neural networks lack work on heterogeneous graph matching. In this study, we recognize that different types of nodes have varying semantics and importance and thus design heterogeneous graph isomorphic network to learn richer representations. More importantly, we propose a two-tier (graph-level and node-level) matching by discriminating node types. Extensive experiments demonstrate the superior matching performance of the proposed framework and the effectiveness of its individual components. Future work will explore the impact of edge semantics on matching and conduct the improvement by distinguishing the importance of element types.

\section*{Acknowledgement}
This research was supported by the National Natural Science Foundation of China (Grant No.62476137) and the Natural Science Foundation of the Jiangsu Higher Education Institutions of China (Grant No.21KJB520017).


\begin{thebibliography}{10}
\expandafter\ifx\csname url\endcsname\relax
  \def\url#1{\texttt{#1}}\fi
\expandafter\ifx\csname urlprefix\endcsname\relax\def\urlprefix{URL }\fi
\expandafter\ifx\csname href\endcsname\relax
  \def\href#1#2{#2} \def\path#1{#1}\fi

\bibitem{bai2019simgnn}
Y.~Bai, H.~Ding, S.~Bian, T.~Chen, Y.~Sun, W.~Wang, Simgnn: A neural network approach to fast graph similarity computation, in: Proceedings of the 12th ACM International Conference on Web Search and Data Mining, 2019, pp. 384--392.

\bibitem{li2019graph}
Y.~Li, C.~Gu, T.~Dullien, O.~Vinyals, P.~Kohli, Graph matching networks for learning the similarity of graph structured objects, in: Proceedings of the 36th International Conference on Machine Learning, 2019, pp. 3835--3845.

\bibitem{ling2021multilevel}
X.~Ling, L.~Wu, S.~Wang, T.~Ma, F.~Xu, A.~X. Liu, C.~Wu, S.~Ji, Multilevel graph matching networks for deep graph similarity learning, IEEE Transactions on Neural Networks and Learning Systems 34~(2) (2023) 799--813.

\bibitem{jin2022cgmn}
D.~Jin, L.~Wang, Y.~Zheng, X.~Li, F.~Jiang, W.~Lin, S.~Pan, Cgmn: A contrastive graph matching network for self-supervised graph similarity learning, in: Proceedings of the 31st International Joint Conferences on Artificial Intelligence, 2022, pp. 2101--2107.

\bibitem{tan2023exploring}
W.~Tan, X.~Gao, Y.~Li, G.~Wen, P.~Cao, J.~Yang, W.~Li, O.~R. Zaiane, Exploring attention mechanism for graph similarity learning, Knowledge-Based Systems (2023) 110739.

\bibitem{kipf2016semi}
T.~N. Kipf, M.~Welling, Semi-supervised classification with graph convolutional networks, in: Proceedings of the International Conference on Learning Representations, 2017.

\bibitem{xu2018powerful}
K.~Xu, W.~Hu, J.~Leskovec, S.~Jegelka, How powerful are graph neural networks?, in: Proceedings of the International Conference on Learning Representations, 2018.

\bibitem{velivckovic2018graph}
P.~Veli{\v{c}}kovi{\'c}, G.~Cucurull, A.~Casanova, A.~Romero, P.~Li{\`o}, Y.~Bengio, Graph attention networks, in: Proceedings of the International Conference on Learning Representations, 2018.

\bibitem{bai2020learning}
Y.~Bai, H.~Ding, K.~Gu, Y.~Sun, W.~Wang, Learning-based efficient graph similarity computation via multi-scale convolutional set matching, in: Proceedings of the AAAI Conference on Artificial Intelligence, Vol.~34, 2020, pp. 3219--3226.

\bibitem{zhuo2022efficient}
W.~Zhuo, G.~Tan, Efficient graph similarity computation with alignment regularization, Advances in Neural Information Processing Systems 35 (2022) 30181--30193.

\bibitem{wang2023contrastive}
L.~Wang, Y.~Zheng, D.~Jin, F.~Li, Y.~Qiao, S.~Pan, Contrastive graph similarity networks, ACM Transactions on the Web 18~(2) (2024) 1--20.

\bibitem{dong2017metapath2vec}
Y.~Dong, N.~V. Chawla, A.~Swami, metapath2vec: Scalable representation learning for heterogeneous networks, in: Proceedings of ACM SIGKDD International Conference on Knowledge Discovery and Data Mining, 2017, pp. 135--144.

\bibitem{wang2019heterogeneous}
X.~Wang, H.~Ji, C.~Shi, B.~Wang, Y.~Ye, P.~Cui, P.~S. Yu, Heterogeneous graph attention network, in: Proceedings of the World Wide Web Conference, 2019, pp. 2022--2032.

\bibitem{yun2019graph}
S.~Yun, M.~Jeong, R.~Kim, J.~Kang, H.~J. Kim, Graph transformer networks, Advances in neural information processing systems 32 (2019).

\bibitem{lv2021we}
Q.~Lv, M.~Ding, Q.~Liu, Y.~Chen, W.~Feng, S.~He, C.~Zhou, J.~Jiang, Y.~Dong, J.~Tang, Are we really making much progress? revisiting, benchmarking and refining heterogeneous graph neural networks, in: Proceedings of ACM SIGKDD Conference on Knowledge Discovery \& Data Mining, 2021, pp. 1150--1160.

\bibitem{zhang2019heterogeneous}
C.~Zhang, D.~Song, C.~Huang, A.~Swami, N.~V. Chawla, Heterogeneous graph neural network, in: Proceedings of the 25th ACM SIGKDD International Conference on Knowledge Discovery \& Data Mining, 2019, pp. 793--803.

\bibitem{wang2023enabling}
J.~Wang, Y.~Guo, L.~Yang, Y.~Wang, Enabling homogeneous gnns to handle heterogeneous graphs via relation embedding, IEEE Transactions on Big Data 9~(6) (2023) 1697--1710.

\bibitem{li2023semi}
L.~Li, M.~Yan, Z.~Tao, H.~Chen, X.~Wu, Semi-supervised graph pattern matching and rematching for expert community location, ACM Transactions on Knowledge Discovery from Data 17~(1) (2023) 1--26.

\bibitem{zhu2022molecular}
X.~Zhu, Y.~Shen, W.~Lu, Molecular substructure-aware network for drug-drug interaction prediction, in: Proceedings of the 31st ACM International Conference on Information \& Knowledge Management, 2022, pp. 4757--4761.

\bibitem{cardoso2020collection}
C.~Cardoso, R.~T. Sousa, S.~K{\"o}hler, C.~Pesquita, A collection of benchmark data sets for knowledge graph-based similarity in the biomedical domain, Database 2020 (2020) baaa078.

\bibitem{coupry2022application}
D.~E. Coupry, P.~Pogany, Application of deep metric learning to molecular graph similarity, Journal of Cheminformatics 14~(1) (2022) 1--12.

\bibitem{xu2017neural}
X.~Xu, C.~Liu, Q.~Feng, H.~Yin, L.~Song, D.~Song, Neural network-based graph embedding for cross-platform binary code similarity detection, in: Proceedings of the 2017 ACM SIGSAC conference on computer and communications security, 2017, pp. 363--376.

\bibitem{dai2023study}
L.~Dai, A study on the application of graph neural network in code clone detection: Improving the performance of code clone detection through graph neural networks and attention mechanisms, in: Proceedings of the International Conference on Networks, Communications and Information Technology, 2023, pp. 172--176.

\bibitem{noble2003graph}
C.~C. Noble, D.~J. Cook, Graph-based anomaly detection, in: Proceedings of the ninth ACM SIGKDD international conference on Knowledge discovery and data mining, 2003, pp. 631--636.

\bibitem{wang2019heterogeneousIJCAI}
S.~Wang, S.~Y. Philip, Heterogeneous graph matching networks: Application to unknown malware detection, in: 2019 IEEE International Conference on Big Data (Big Data), IEEE, 2019, pp. 5401--5408.

\bibitem{ebsch2020using}
C.~L. Ebsch, J.~A. Cottam, N.~C. Heller, R.~D. Deshmukh, G.~Chin, Using graph edit distance for noisy subgraph matching of semantic property graphs, in: 2020 IEEE international conference on big data (big data), IEEE, 2020, pp. 2520--2525.

\bibitem{qi2021unsupervised}
Z.~Qi, Z.~Zhang, J.~Chen, X.~Chen, Y.~Xiang, N.~Zhang, Y.~Zheng, Unsupervised knowledge graph alignment by probabilistic reasoning and semantic embedding, arXiv preprint arXiv:2105.05596 (2021).

\bibitem{wu2020neighborhood}
Y.~Wu, X.~Liu, Y.~Feng, Z.~Wang, D.~Zhao, Neighborhood matching network for entity alignment, in: Proceedings of the 58th Annual Meeting of the Association for Computational Linguistics, 2020, pp. 6477--6487.

\bibitem{pelillo1999matching}
M.~Pelillo, K.~Siddiqi, S.~W. Zucker, Matching hierarchical structures using association graphs, IEEE Transactions on Pattern Analysis and Machine Intelligence 21~(11) (1999) 1105--1120.

\bibitem{sarlin2020superglue}
P.-E. Sarlin, D.~DeTone, T.~Malisiewicz, A.~Rabinovich, Superglue: Learning feature matching with graph neural networks, in: Proceedings of the IEEE/CVF conference on computer vision and pattern recognition, 2020, pp. 4938--4947.

\bibitem{guo2018neural}
M.~Guo, E.~Chou, D.-A. Huang, S.~Song, S.~Yeung, L.~Fei-Fei, Neural graph matching networks for fewshot 3d action recognition, in: Proceedings of the European Conference on Computer Vision, 2018, pp. 653--669.

\bibitem{wang2021neural}
R.~Wang, J.~Yan, X.~Yang, Neural graph matching network: Learning lawler’s quadratic assignment problem with extension to hypergraph and multiple-graph matching, IEEE Transactions on Pattern Analysis and Machine Intelligence 44~(9) (2021) 5261--5279.

\bibitem{riesen2013novel}
K.~Riesen, S.~Emmenegger, H.~Bunke, A novel software toolkit for graph edit distance computation, in: Graph-Based Representations in Pattern Recognition: 9th IAPR-TC-15 International Workshop, 2013, pp. 142--151.

\bibitem{bunke1983distance}
H.~Bunke, What is the distance between graphs, Bulletin of the EATCS 20 (1983) 35--39.

\bibitem{bunke1998graph}
H.~Bunke, K.~Shearer, A graph distance metric based on the maximal common subgraph, Pattern Recognition Letters 19~(3-4) (1998) 255--259.

\bibitem{borgwardt2005shortest}
K.~M. Borgwardt, H.-P. Kriegel, Shortest-path kernels on graphs, in: Proceedings of the 5th IEEE International Conference on Data Mining, 2005, pp. 74--81.

\bibitem{yan2005substructure}
X.~Yan, P.~S. Yu, J.~Han, Substructure similarity search in graph databases, in: Proceedings of the 2005 ACM SIGMOD international conference on Management of data, 2005, pp. 766--777.

\bibitem{yoshida2019learning}
T.~Yoshida, I.~Takeuchi, M.~Karasuyama, Learning interpretable metric between graphs: Convex formulation and computation with graph mining, in: Proceedings of the 25th ACM SIGKDD International Conference on Knowledge Discovery \& Data Mining, 2019, pp. 1026--1036.

\bibitem{neuhaus2006fast}
M.~Neuhaus, K.~Riesen, H.~Bunke, Fast suboptimal algorithms for the computation of graph edit distance, in: Structural, Syntactic, and Statistical Pattern Recognition: Joint IAPR International Workshops, SSPR 2006 and SPR 2006, Hong Kong, China, August 17-19, 2006. Proceedings, Springer, 2006, pp. 163--172.

\bibitem{riesen2009approximate}
K.~Riesen, H.~Bunke, Approximate graph edit distance computation by means of bipartite graph matching, Image and Vision computing 27~(7) (2009) 950--959.

\bibitem{fankhauser2011speeding}
S.~Fankhauser, K.~Riesen, H.~Bunke, Speeding up graph edit distance computation through fast bipartite matching, in: Graph-Based Representations in Pattern Recognition: 8th IAPR-TC-15 International Workshop, 2011, pp. 102--111.

\bibitem{blumenthal2020exact}
D.~B. Blumenthal, J.~Gamper, On the exact computation of the graph edit distance, Pattern Recognition Letters 134 (2020) 46--57.

\bibitem{schlichtkrull2018modeling}
M.~Schlichtkrull, T.~N. Kipf, P.~Bloem, R.~Van Den~Berg, I.~Titov, M.~Welling, Modeling relational data with graph convolutional networks, in: Proceedings of the 15th Extended Semantic Web Conference, 2018, pp. 593--607.

\bibitem{spearman1904proof}
C.~Spearman, The proof and measurement of association between two things, The American Journal of Psychology 15~(1) (1904) 72--101.

\bibitem{kendall1938new}
M.~G. Kendall, A new measure of rank correlation, Biometrika 30~(1/2) (1938) 81--93.

\bibitem{chang2022accelerating}
L.~Chang, X.~Feng, K.~Yao, L.~Qin, W.~Zhang, Accelerating graph similarity search via efficient ged computation, IEEE Transactions on Knowledge and Data Engineering 35~(5) (2022) 4485--4498.

\bibitem{liu2022graph}
X.~Liu, Y.~Song, Graph convolutional networks with dual message passing for subgraph isomorphism counting and matching, in: Proceedings of the AAAI Conference on Artificial Intelligence, Vol.~36, 2022, pp. 7594--7602.

\bibitem{roy2022interpretable}
I.~Roy, V.~S. B.~R. Velugoti, S.~Chakrabarti, A.~De, Interpretable neural subgraph matching for graph retrieval, in: Proceedings of the AAAI conference on artificial intelligence, Vol.~36, 2022, pp. 8115--8123.

\bibitem{chang2020speeding}
L.~Chang, X.~Feng, X.~Lin, L.~Qin, W.~Zhang, D.~Ouyang, Speeding up ged verification for graph similarity search, in: 2020 IEEE 36th International Conference on Data Engineering (ICDE), IEEE, 2020, pp. 793--804.

\bibitem{EJ2024ged}
E.~Jain, I.~Roy, S.~Meher, S.~Chakrabarti, A.~De, Graph edit distance with general costs using neural set divergence, in: Proceedings of the 38th Annual Conference on Neural Information Processing Systems, 2024.

\bibitem{zheng2024grasp}
H.~Zheng, J.~Shi, R.~Yang, Grasp: Simple yet effective graph similarity predictions, in: Proceedings of the AAAI Conference on Artificial Intelligence, 2025.

\bibitem{bai2021tagsim}
J.~Bai, P.~Zhao, Tagsim: type-aware graph similarity learning and computation, Proc. VLDB Endow. 15~(2) (2021) 335–347.

\end{thebibliography}
\end{document}